\documentclass{article}
\usepackage{graphicx}
\usepackage{epsfig}
\usepackage{amsmath,amsthm, amssymb}
\usepackage{url}

\title{An Inductive Formalization of Self Reproduction in Dynamical Hierarchies}
\author{Janardan Mishra \\
janardanmisra@acm.org}

\begin{document}

\maketitle

\begin{abstract}

Formalizing self reproduction in dynamical hierarchies is one of the important problems in Artificial Life (AL) studies. We study, in this paper, an inductively defined algebraic framework for self reproduction on macroscopic organizational levels under dynamical system setting for simulated AL models and explore some existential results. Starting with defining self reproduction for atomic entities we define self reproduction with possible mutations on higher organizational levels in terms of hierarchical sets and the corresponding inductively defined `meta' - reactions. We introduce constraints to distinguish a collection of entities from genuine cases of emergent organizational structures.

\end{abstract}

\section{Introduction}

Self reproduction is one of ubiquitously studied phenomena in
Artificial Life (ALife) studies. There are early models of self
reproduction based on cellular automata and their modern
simplified versions as well as other models with novel syntactical
representations and corresponding semantics~\cite{Sipper98}. There
exist formalization aimed at various levels of abstractions and
properties for self reproduction. Recent work on formalization
include~\cite{AL03} where authors define a probability measure to
quantify how much probable is self reproduction of a subsystem in
a model under one environment with respect to some other
environment.

Nonetheless we lack complete understanding of how self
reproduction emerges and maintains itself across higher level
organizational structures. Real life is full of examples of such
higher order structures -- starting with simple molecules,
monomers, polymers, supra molecular structures like proteins,
organelles, cells, organisms.

In~\cite{mr98} authors present a $2 D$ lattice automaton based
simulation of higher order emergent structure (upto 3rd order
hyperstructure - micelle) in the sprit of actual physical
dynamics. They have also presented an analytz following the
formalism of hyperstructures to explain their
synthesis~\cite{bass92,ana011,ana013,ana012}. Nonetheless, the
hyperstructure based approach for dynamical hierarchies leave some
critical aspects informally defined (e.g., emergent properties,
observation process), is semi formal in nature, and thus allow
trivial cases~\cite{ana013,ana012}. In contrast, we adopt in this
paper a more formal approach based upon the set and graph
theoretic notions while precisely working with the simulations of
models (see ``weak emergence"~\cite{ac:BMPRAGIKR00}).

\section{The Framework}

\newtheorem{df}{Definition}
\newtheorem{Th}{Theorem}

The following basic definitions for multisets will be used in the
paper:

A \emph{multiset} $M$ on a set $E$ is a mapping associating
non-negative integers with each element of $E$, that is, $M: E
\rightarrow \mathcal{N}$, where $\mathcal{N} = \{0, 1, \ldots\}$.
For $e \in E$, $M(e)$ is called its \emph{multiplicity} in the
multiset.

Set of all elements $e \in E$ with nonzero multiplicity is called
the \emph{support} of $M$, which is denoted as $Supp(M) = \{e \in
E \mid M(e) > 0\}$.

For multisets $M$ and $M'$ on $E$, we define $M \cup M': E
\rightarrow \mathcal{N}$ such that $\forall e \in E. (M \cup
M')(e)= M(e)+M'(e)$. Similarly $(M \cap M')(e)= min(M(e),
M'(e))$.\\

We will use the term artificial chemistry (AC) in a generic sense
applicable to a wide class of ALife models with computational
dynamics. ACs represent a mathematically generalized metaphors of
`` real chemistry" with well defined ``laws of interaction or
reaction semantics" between the ``elements" or ``molecules" of the
model universe. A detailed review of ACs also appears
in~\cite{ac:Dittrich01}.

An AC $\mathcal{A}$ is usually started with an initial population
of a multiset of molecules $P_0$. $\mathcal{A}$ evolves over time
and we have different populations consisting of different
multisets of molecules during the course of evolution. We
represent time progression of $\mathcal{A}$ as an infinite
sequence of multisets $\mathcal{P}$ = $<P_0, P_1, \ldots>$ such
that $P_i$ precedes $P_{j}$ $\forall j>i$. Multisets $P_0, P_1,
\ldots$ are also referred to as \emph{states} of $\mathcal{A}$ and
$\mathcal{P}$ is called a \emph{run} or \emph{simulation} of
$\mathcal{A}$. A finite strictly consecutive subsequence of states
of $\mathcal{A}$, $<P_{i}, P_{i+1}, \ldots, P_{i+n}>$ is termed as
as a partial run of $\mathcal{P}$. A non consecutive subsequence
$<P_{i_1}, P_{i_2}, \ldots, P_{i_n}>$ with $P_{i_j}$ precedes
$P_{i_k}$ $\forall i_k > i_j$ is called subsequence of states in
$\mathcal{P}$. The set of all such different runs of $\mathcal{A}$
is denoted by $\Gamma$. Each run of $\mathcal{A}$ has potentially
infinite states, though in case of cycles there will be repeating
sub sequences of states.

%Two runs $\mathcal{P}$ = $<P_0, P_1, \ldots>$ and $\mathcal{P'}$ =
%$<P'_0, P'_1, \ldots>$ of $\mathcal{A}$ are \emph{asymptotically
%same} (written as $\mathcal{P} \sim \mathcal{P'}$) \emph{iff
%}$\exists$ $n>0$ such that $P_k = P'_k$ $\forall k >n$. The
%special case when all the runs of $\mathcal{A}$ are asymptotically
%same, it is called \emph{ergodic}.

We assume in this paper that reaction semantics defined in the ALife
model is deterministic. More general case of probabilistic or
stochastic reactions will be dealt in future. For a reaction $r$
defined in terms of the inputs and the corresponding outputs
(ignoring other conditions), we define $\mathit{input_r}$ as a
multiset of input molecules and $\mathit{output_r}$ as the
multiset of the outputs of $r$. For a sequence of reactions
$\mathbf{r}$ = $<r_1, r_2, \ldots, r_n>$, we define
$\mathit{Input}(\mathbf{r})$ = $\bigcup_{r_j \in \mathbf{r}}
\mathit{input_{r_j}}$ as the multiset of all participating input
molecules in the reaction sequence. Similarly
$\mathit{Output}(\mathbf{r})$ = $\bigcup_{r_j \in \mathbf{r}}
\mathit{output_{r_j}}$ is the multiset of all output molecules in
$\mathbf{r}$.

%Note that every reaction in the AGC is node/atompreserving, that is,
%the number of atoms before and after the reaction are the same.\\

% Definition#1

\begin{df}[\emph{Feasible Reaction}]  A reaction
$r$ is said to be \emph{feasible} in a state $P_i$ \emph{iff}
$\forall$ $g \in Supp(\mathit{input_r})$, $P_i(g)>
\mathit{input_r(g)}$.
\end{df}

Informally this means reaction $r$ may execute if all required
input molecules are available in the state $P_i$ of $\mathcal{A}$.
If state $P_i$ is in run $\mathcal{P}$ = $<P_0, P_1, \ldots>$ of
$\mathcal{A}$, $r$ is also said to be a feasible reaction in run
$\mathcal{P}$. In actual ALife models there can be many other global
conditions or environmental constraints associated with the
feasibility of reactions defined by the designer. We have though
ignored, these can be added without much difficulty when applying
the framework on these models. Note that feasibility of a reaction
does not imply automatically that it will be executed as well
since that depends on the scheduling algorithm defined by the
designer of the chemistry which selects the reactions to execute
at any state of the chemistry.

We can extend above definition  to a sequence of reactions as
follows: define $\mathbf{r}$ = $<r_1, r_2, \ldots, r_n>$ as a
\emph{feasible reaction sequence} in a state subsequence
$<P_{i_1}, P_{i_2} \ldots P_{i_n}>$ of run $\mathcal{P}$
\emph{iff} $r_j$ is a feasible reaction in state $P_{i_j}$
$\forall j$ = $1, 2, \ldots n$. $\mathbf{r}$ is also called a
\emph{feasible reaction sequence} in the run $\mathcal{P}$.

% Definition#2

\begin{df}[Potential Causality] Let $\Omega_0$ be the set of
molecules for a run $\mathcal{P}$ $\in \Gamma$ of $\mathcal{A}$.
We define potential causal relation $\Rightarrow$ between two
molecules as follows: $\Rightarrow \subseteq \Omega_0 \times
\Omega_0$ such that for molecules $g_1$, $g_2$ $\in$ $\Omega_0$,
$(g_1, g_2) \in \Rightarrow$ if and only if $\exists$ feasible
reaction $r$ in the run $\mathcal{P}$ such that $g_1 \in
\mathit{input_r}$ and $g_2 \in \mathit{output_r}$. $r$ is termed
as \emph{potential causal link} between $g_1$ and $g_2$ and we
represent this using $g_1 \Rightarrow_r g_2$.\end{df}

There can be multiple causal links between two molecules and each
molecule can be causally linked to multiple other molecules. We
can also define a multi-graph using all the potential causal
reactions at any state of the chemistry.

\begin{df}[Potential Reaction Graph] Define a multi-graph $G_i=(V_i,
E_i)$ for the state $P_i$ such that for each molecule $g \in
Supp(P_i)$ there is a node $e_g$ in $E_i$ and if $g \Rightarrow_r
g'$ then we have a directed edge from $e_{g}$ to $e_{g'}$ in $E_i$
with label $r$, where $r$ is feasible in $P_i$.
\end{df}

Next consider a feasible reaction sequence $\mathbf{r}$ =  $<r_1,
r_2, \ldots, r_n>$ and define for $g, g' \in \Omega_0$, $g
\Rightarrow_{\mathbf{r}} g'$ when $\exists$ $g_0, g_1, \ldots,
g_n$ such that $g_0 = g$, $g_n = g'$ and $g_i \Rightarrow_{r_i}
g_{i+1}$ $\forall$ $0 \leq i < n$. Such a feasible reaction
sequence $\mathbf{r}$ will be termed as a \emph{potential causal
path} between $g, g'$. There can be multiple such potential causal
paths present between $(g, g')$. Note that potential causal path
is not a path in potential reaction graph but is constructed when
chemistry evolves over time.

Now we can define potential self replication using the principle
of preservation of overall resources (dilution flux) along with
the concept of potential causality.

\begin{df}[\textbf{Potentially Self Reproducing Entities}]
A molecule/entity $g$ is defined as potentially self reproducing
in a chemistry $\mathcal{A}$, if $\exists$ a run $\mathcal{P}$ of
$\mathcal{A}$ for which the following holds:

$\exists$  $g' \in \Omega_0$  such that the following conditions
are satisfied:
\begin{description}
    \item[Observational Equivalence]$g' \sim g$, where exact definition
of $\sim \subseteq \Omega_0 \times \Omega_0$ is dependent on the
underlying chemistry and its designer or the observer. For example
if molecules are represented as graphs then $\sim$ can be defined
as graph isomorphism, or if molecules are strings then it will be
character by character string equivalence. $\sim$ can even be
defined by the designer as functional equivalence.
    \item[Reflexive Autocatalysis]
    $\exists$ feasible sequence of reactions $\mathbf{C_p}$ so that $g
\Rightarrow_{\mathbf{C_p}} g'$, that is, $\mathbf{C_p}$ is a
potential causal path between $g$ and $g'$.
    \item[Material Basis]For every such potential causal path $C_p$
between $(g, g')$, which is a feasible reaction sequence in a
partial run $<P_{i_1}, P_{i_2} \ldots P_{i_n}>$ of $\mathcal{P}$,
we have $P_{i_n}(g) > P_{i_1}(g)$ and $\exists$ $X \subseteq
\mathit{Input}(C_p) - \{g\}$, ($X \neq \emptyset$) such that
$\forall$ $g_x \in X$ $P_{i_n}(g_x) < P_{i_1}(g_x)$. \\
Informally, this states that there should be an increase in the
size of population of $g$ and corresponding decrease in some other
populations of participating entities ($X$) in the state $P_{i_n}$
as compared to the sizes of these populations in initial state
$P_{i_1}$.
\end{description}
\end{df}

Let me now discuss the above conditions in the context of ALife
studies: the first requirement of observational equivalence is
fundamental to any ALife study because otherwise in the model
universe itself there cannot have some fundamental embodied
equivalence between two entities and therefore always some
external observer is needed who imposes the equivalence ($\sim$)
between the molecule $g$ and the product $g'$ to define self
reproduction. The apparently objective alternatives to this view
where one might consider structural or functional equivalences can
themselves be considered as externally imposed criterion not
inherent in the model universe unless the underlying chemistry
evolves or possesses some kind of structural or functional
recognition capability. For most of the ALife studies, it is upon the
observer or the designer to define the recognition process which
can be used to determine the equivalence between molecules $g$ and
$g'$. This can also be seen in light of the Valera's theory of
autopoisis which emphasizes upon the ``emergence" of autonomy in
life forms~\cite{Zeleny81}. Also note that by equivalence we may
not require that $g$ and $g'$ are identical and thus $g$ can
reproduce with mutations under some observable limit.

The second requirement of reflexive autocatalysis should be
obvious since all molecules not present in the chemistry at the
beginning should be the result of some reactions. Reflexive
autocatalysis denotes one or more reaction steps in the reaction
sequence starting from $g$ and yielding another molecule $g'$
finally, which should be observationally equivalent to $g$.

The last requirement of material basis is to capture the essence
of entity - environment interaction quantitatively along the lines
of real chemistry. This condition dictates that new molecule
appearing in the chemistry must not be the result of some sort of
magical appearance out of nothing. This requirement is most often
ignored in ALife studies and alternately weakly captured by imposing
dilution flux which keeps the volume of the chemistry constant.
Our formulation makes clear connection between the transformation
of reacting molecules as per the reactions in $\mathbf{C_p}$.

Each potential causal path $\mathbf{C_p}$ leading to potential
self replication for $g$ is also called \emph{potential self
reproducing path of $g$}. Note that potential self replication
does not necessarily guarantee that self replication of $g$ will
occur in every run in which $C_p$ is potentially feasible. The
only thing which is guaranteed is that there exists at least one
run of $\mathcal{A}$, where $C_p$ will actually execute and thus
lead to self replication of $g$. This further highlights the
importance of emergence of \emph{membrane structures} in real life
which had very profound role in making potential self reproducing
paths actual execution paths since due to the presence of membrane
boundaries these potential self reproducing reactions could
actually execute with high probability.

Furthermore it is not again guaranteed that in all those runs
where $C_p$ executes, there is no spontaneous emergence of same
entity $g$ in some other way not involving $g$ in the reactions.
Indeed this is bit unfortunate because then in that case it will
not be possible for any outside observer to establish reflexive
autocatalysis just by looking at entities at different states of
the chemistry.

Next we will consider more strict characterization self
reproduction for special class of chemistries which employ
sequential scheduling where at any state of the chemistry during
simulation only one reaction is selected for the execution. For
these chemistries we consider cyclic runs and prove that every
potentially self reproducing entity indeed self reproduces.
%Although possibility of spontaneous emergence of same entity
%cannot be ruled out even in this case.

% Definition#3

\begin{df}[\emph{Cyclic Run}] A run $\mathcal{P}$ = $<P_0, P_1,
\ldots>$ of $\mathcal{A}$ is \emph{cyclic iff} $\exists n \geq 0,
l > 0$ such that $\forall k \geq 1$, $0 < r \leq l$, $P_{n + kl
+r}$ = $P_{n+r}$. Subsequence $<P_{n+1}, \ldots P_{n+l}>$ is the
\emph{cycle} in $\mathcal{P}$ and a cyclic run is therefore
represented as $\mathcal{P}$ = $<P_0, P_1, \ldots P_n, [P_{n+1},
\ldots P_{n+l}]^\infty>$.
\end{df}

\begin{Th} For a cyclic run $\mathcal{P}$ = $<P_0, P_1, \ldots
P_n, [P_{n+1}, \ldots P_{n+l}]^\infty >$ of $\mathcal{A}$, a
potentially self reproducing entity $g \in \Omega_0$ actually self
reproduces if $\exists$ potential self reproducing path $C_p$ for
$E$ which is feasible in the cycle $\mathbf{p}=[P_{n+1}, \ldots
P_{n+l}]$.
\end{Th}

\begin{proof} This is because the feasible reaction
sequence $C_p$ indeed executes in the cycle $\mathbf{p}$,
otherwise there will be different states of the chemistry not
present in $\mathbf{p}$ because of the execution of some other
reactions not in $C_p$, contradicting the very structure of the
cycle. Furthermore due to sequential scheduling of the reactions
during simulations there is always only one potential reaction
which is executed in every state of the cycle.
\end{proof}

Though above characterizations only specify self replication of a
single molecule, it can be seamlessly extended to the case of
simultaneous self replication of multiple molecules. In such cases
either scheduling algorithm will have to execute several reactions
in parallel or the potential self reproducing paths for several
molecules might be intermixed with each other.

Next we will discuss an important extension to above definitions
to handle more realistic scenarios whereby sets of molecules
forming higher level organizational structures reproduce
collectively.

\section{Self Reproduction on Higher Organizational Levels}

\subsection{Entities on Higher Organizational Levels}

To achieve this aim, we will inductively define the hierarchical
sets as entities at different levels. Consider the \emph{level}
$0$ entities as all ``syntactically valid" molecules appearing at
any state of the chemistry during its dynamical progression
through time. $\Omega_0$ used above denotes the set of all such
\emph{level} $0$ entities.

Then \emph{level} $1$ entities are any finite subsets of
$\Omega_0$ of size $>1$. Let $\Omega_1$ be the set of all such
\emph{level} $1$ entities. Thus

\[\Omega_1 = \{ x \ | \ [x \subseteq \Omega_0]
             \  \wedge \ [x \cap \Omega_0 \neq \emptyset]\ \wedge \ [|x| > 1]
             \}\]

Note that we do not consider a singleton set consisting of only
one \emph{level} $0$ entity as an \emph{level} $1$ entity.
Similarly \emph{Level} $2$ entities consist of finite number of
\emph{level} $0$ and \emph{level} $1$ entities. That is, each
\emph{level} $2$ entity is finite subset of $\Omega_0 \cup
\Omega_1$ of size $> 1$. Let $\Omega_2$ be set of all such
\emph{level} $2$ entities. This way we can inductively define the
set of \emph{level} $n$ entities as
\[\Omega_n = \{ x \ | \ [x \subseteq \bigcup_{0 \leq i < n}\Omega_i]
                \  \wedge \ [x \cap \Omega_{n-1} \neq \emptyset]\ \wedge \ [|x| >
                1]\}\]

The above classification of higher level entities in the
chemistry, though captures syntactical essence of hierarchical
structures, does not specify their dynamical structure, which is
one of the important problems to be addressed in ALife theories. In
this paper we will focus our attention to only the
characterization of self replication for such higher level
structures and will not provide analysis on how these structures
emerge per se in the chemistry and maintain themselves.

\subsection{Defining Meta Reactions}
I will proceed by defining higher level ``meta" reactions which
form the counter part of higher level entities defined above.

Let us consider a \emph{level} $1$ entity $\zeta_{1} = \{e_1, e_2,
\ldots, e_r \}\in \Omega_1$, where $e_i \in \Omega_0$ and a
feasible reaction sequence $R$ = $<r_1, r_2, \ldots, r_k>$,
satisfying the following:

$$\forall 1 \leq i \leq k. \mathit{input_{r_i}} \cap \zeta_{1} \neq \emptyset$$

Then in that case we say that (\emph{level} $1$ entity)
$\zeta_{1}$ \emph{takes part} in \emph{level} $1$ (meta) reaction
$R$.

Also consider some other \emph{level} $1$ entity $\zeta_{2}$ such
that $$\zeta_{2} \subseteq (\mathit{Output}(R) -
\mathit{Input}(R))$$ Then in that case we say that $\zeta_{2}$ is
\emph{potentially causally related} to $\zeta_{1}$ and write it as
$\zeta_{1} \Rightarrow^1_R \zeta_{2}$.

For example, consider a sequence of reactions feasible in three
consecutive states of a chemistry as $$R = <r_1: a + a_1
\rightarrow 2c, r_2: a + c \rightarrow d, r_3: e + e_1 \rightarrow
d + f>$$ Now we can define $\zeta_1 = \{a, c, e_1\}$ which takes
part in $R$ because $\mathit{input_{r_1}} = \{a, a_1\}$ and $\{a,
a_1\} \cap \zeta_1 = \{a\} \neq \emptyset$, similarly
$\mathit{input_{r_2}} \cap \zeta_1 \neq \emptyset$ and
$\mathit{input_{r_3}} \cap \zeta_1 \neq \emptyset$. If we consider
$\zeta_2 = \{c, d, f\}$ then $\zeta_{2} \subseteq
(\mathit{Output}(R) - \mathit{Input}(R))$ Therefore we can also
infer that $\zeta_{2}$ is potentially causally related to
$\zeta_{1}$ through $R$, i.e., $\{a,b,c\} \Rightarrow^1_R
\{c,d,f\}$. Note that the given formulation also allows trivial
cases where certain collections of entities are inferred as
causally connected while in reality only the individual elements
appearing in those collections are independently causally
connected. For illustration let me consider another feasible
sequence $$R' = <r_1: a + a_1 \rightarrow 2a', r_2: b + b_1
\rightarrow b', r_3: e + e_1 \rightarrow e' + f>$$ Also define
$\zeta' = \{a, b, e_1\}$ which takes part in $R$. Next let us
select $\zeta'' = \{a', b', f\}$ then $\zeta'' \subseteq
(\mathit{Output}(R') - \mathit{Input}(R'))$. Therefore we can
infer that $\zeta' \Rightarrow^1_{R'} \zeta''$ even though this is
merely because of the fact that component elements in $\zeta'$ and
$\zeta''$ are independently causally connected, i.e., $a
\Rightarrow_c a'$ through $r_1$, $b \Rightarrow_c b'$ through
$r_2$, and $e \Rightarrow_c f$ through $r_3$. It is clear that, to
be meaningful, we need to exclude such trivial cases while
defining potential self replication for emerging higher level
entities.

\textbf{Constraint of non-triviality:} This is done by enforcing
another constraint to ensure that total number of potential causal
paths between $\zeta'$ and $\zeta''$ are strictly more than
$|\zeta''|$ - this is because - then in that case there will be at
least one component in $\zeta''$ which must be causally connected
to more than one element in $\zeta'$. An even more strict
constraint using the concept of reaction graphs can be formulated
where we can demand absence of cliques in the reaction graph
consisting of potential causal paths between elements of $\zeta'$
and $\zeta''$ to ensure non triviality of causality but we will
not pursue it here.

Also it should be pointed out that \emph{level} $1$ reactions have
to have time progression built into them, that is, should be
feasible reaction sequences. Thus not every subset of \emph{level}
$0$ reactions can be considered as a \emph{level} $1$ reaction.

Now we are in a position to define a potential causal path which
will be then used to define self replication of \emph{level} $1$
entities in terms of \emph{level} $1$ reactions.

Consider two such \emph{level} $1$ reactions $R$ = $<r_1, r_2,
\ldots, r_n>$ and $S$ = $<s_1, s_2, \ldots, s_m>$. We say $R$
\emph{temporally precedes} $S$ if and only if $r_1$ precedes $s_1$
and $r_n$ also precedes $s_m$ over some sequence of states
$<P_{i_i}, P_{i_2}, \ldots, P_{i_k}>$ in the run $\mathcal{P}$,
where $max(m,n) \leq k \leq m + n$. Then $<R, S>$ can be
considered as a \emph{level} $1$ feasible reaction sequence.

Let $\mathbf{R}$ = $<R_1, R_2, \ldots, R_n>$ and define for
$\zeta, \zeta' \in \Omega_1$, $ \zeta \Rightarrow^1_{\mathbf{R}}
\zeta'$ when $\exists$ $\zeta_0, \zeta_1, \ldots, \zeta_n \  \in \
 \Omega_1$ such that $\zeta_0 = \zeta$, $\zeta_n = \zeta'$ and $\zeta_i
\Rightarrow^1_{R_i} \zeta_{i+1}$ $\forall$ $0 \leq i < n$. Such
feasible reaction sequence $\mathbf{R}$ will be termed as the
\emph{level 1 potential causal path} between $\zeta$ and $\zeta'$.
There can be multiple such potential causal paths present between
$(\zeta, \zeta')$.

As discussed before, this definition permits trivial scenario of
\emph{level $1$} potential causal paths which are the result of
the presence of independent \emph{level $0$}  potential causal
paths between the elements of $\zeta = \{e_1, e_2, \ldots, e_l\}$
and $\zeta' = \{e_1', e_2', \ldots, e_m'\}$. In order to eliminate
this situation we need to enforce the constraint of non-triviality
: we say $\mathbf{R}$ is a \emph{non-trivial potential causal
path} between $\zeta$ and $\zeta'$ if and only if number of
potential causal paths between pairs of elements from $\zeta$ and
$\zeta'$ are more than $m$ indicating network dependence. This is
because in case of trivial potential causal path between $\zeta$
and $\zeta'$ there will in turn be exactly $m$ \emph{level $0$}
independent potential causal paths producing each of $e_i'$, $1
\leq i \leq m$.

%Definition#

\begin{df}[\textbf{Potentially Self Reproducing Sets of Entities}]
A \emph{level} $1$ entity $\zeta$ is defined as potentially self
reproducing in chemistry $\mathcal{A}$, if $\exists$ a run
$\mathcal{P}$ of $\mathcal{A}$ for which the following holds:

$\exists$  $\zeta' \in \Omega_1$  such that the following
conditions are satisfied:
\begin{description}
    \item[Observational Equivalence] $\zeta' \sim^1 \zeta$, where exact
definition of $\sim^1 \subseteq \Omega_1 \times \Omega_1$ is again
dependent on the underlying chemistry structure and its designer
or the observer. An observer might, for example, define $\zeta'
\sim^1 \zeta$ if both sets are equivalent under $\sim$, that is,
there exists an one to one equivalence between the elements of
$\zeta$ and $\zeta'$.
    \item[Reflexive Autocatalysis] $\exists$ non trivial
causal path $\mathbf{C_p}$ consisting of level $1$ reactions  so
that $\zeta \Rightarrow^1_{\mathbf{C_p}} \zeta'$.
    \item[Material Basis] For every such non-trivial potential causal path
$C_p$ between $(\zeta, \zeta')$, which is a feasible reaction
sequence in a subsequence of states $<P_{i_1}, P_{i_2} \ldots
P_{i_n}>$ of $\mathcal{P}$, there should be an resultant increase
in the size of population of $\zeta$ and corresponding decrease in
some other entity populations participating the the reaction
sequence ($C_p$) in state $P_{i_n}$ as compared to the sizes of
these populations in initial state $P_{i_1}$.
\end{description}
\end{df}

The above approach can be extended without much difficulty to
inductively define meta reactions on even higher levels  in the
chemistry.

Note that unlike other formalisms based upon
hyperstructures~\cite{bass92,ana011,mr98} we do not reply on
informally defined notion of observation dependent emergent
properties on higher level (hyper) structures but specifically
focus our attention to self reproduction as such property which
emerges owing to collective reaction semantics. The notion of
observational equivalence as discussed before should not be
confused with the notion of emergent properties in
hyperstructures.

Due to space limitations detailed case study illustrating the
formalism would be presented a forthcoming paper~\cite{forth01}.

%\section{Examples}
%
%Next I will consider examples from known studies to illustrate the
%expressive power of the inductive formulation defined above.

\section{Conclusion}
In this paper we presented a rigorous formalism to define higher
level organizational structures in terms of hierarchal sets and
corresponding non trivial meta reactions. The formalism can
adequately capture syntactical representations of important higher
order structures and meta reaction sequences these structures can
take part in. The constraint of non triviality allows us to
distinguish the genuine case of higher level organization with a
collection of reacting entities. The formalism allowed us to
define concretely the case of self reproduction even when we allow
mutations under observable limitations. The definition of self
reproduction is quite generic and captures the essence of self in
terms of observed equivalence.

\section{Further Work}

This is an ongoing work with the aim to capture the necessary and
sufficient conditions for evolution to occur in important ALife
studies. We need to introduce explicitly a notion of mutations,
heredity and most importantly selection by considering a
population of reproducing entities. We need to define certain
closure properties for such higher level entities which will
ensure that even under mutations which change the syntactical
structure of entities they can nonetheless semantically retain
their properties e.g. self reproduction. Detailed case studies
will be used to further refine the formalism. We also need to
extend the current formalism by considering the more generic
scenario involving probabilistic reactions or stochastic dynamics,
whereby we can address the questions involving how do
developmental pathways get selected and fixed over the course of
evolution.

\bibliography{alife}

\newcommand{\etalchar}[1]{$^{#1}$}
\begin{thebibliography}{RBM{\etalchar{+}}01b}

\bibitem[Baa92]{bass92}
N.~Baas.
\newblock Emergence, hierarchies and hyperstructures.
\newblock In {\em Artificial Life III}, pages 515--537. Cambridge, MA: MIT
  Press, 1992.

\bibitem[BH03]{AL03}
Adams B. and Lipson H.
\newblock A universal framework for self-replication.
\newblock In {\em European Conference on Artificial Life, ECAL'03, Lecture
  Notes in Computer Science Vol 2801}, pages 1--9, Dortmund, Germany, 2003.

\bibitem[BMP{\etalchar{+}}00]{ac:BMPRAGIKR00}
M.~A. Bedau, J.~S. McCaskill, N.~H. Packard, S.~Rasmussen, C.~Adami, D.~G.
  Green, T.~Ikegami, K.~Kaneko, and T.~S. Ray.
\newblock Open problems in artificial life.
\newblock {\em Artif. Life}, 6(4):363--376, 2000.

\bibitem[DZB01]{ac:Dittrich01}
P.~Dittrich, Jens Ziegler, and Wolfgang Banzhaf.
\newblock Artificial chemistries - a review.
\newblock {\em Artificial Life}, 7(3), 2001.

\bibitem[GM01]{ana013}
Dominique Gross and Barry McMullin.
\newblock Is it the right ansatz?
\newblock {\em Artificial Life}, 7(4):355 -- 365, 2001.

\bibitem[Mis]{forth01}
Janardan Mishra.
\newblock A multi-set theoretic framework for evolution.
\newblock {\em Forthcoming}.

\bibitem[MS98]{mr98}
B.~Mayer and Rasmussen S.
\newblock Self-reproduction of dynamical hierarchies in chemical systems.
\newblock In C.~Adami, R.~Belew, H.~Kitano, and C.~Taylor, editors, {\em
  Artificial life VI}, page 123–129. Cambridge, MA: MIT Press, 1998.

\bibitem[RBM{\etalchar{+}}01a]{ana011}
Steen Rasmussen, Nils~A. Baas, Bernd Mayer, Martin Nilsson, and Michael~W.
  Olesen.
\newblock Ansatz for dynamical hierarchies.
\newblock {\em Artificial Life}, 7(4):329 -- 353, 2001.

\bibitem[RBM{\etalchar{+}}01b]{ana012}
Steen Rasmussen, Nils~A. Baas, Bernd Mayer, Martin Nilsson, and Michael~W.
  Olesen.
\newblock A defense of the ansatz for dynamical hierarchies.
\newblock {\em Artificial Life}, 7(4):367--373, 2001.

\bibitem[Sip98]{Sipper98}
Moshe Sipper.
\newblock Fifty years of research on self-replication: An overview.
\newblock In {\em Artificial Life IV}, pages 237--257, 1998.

\bibitem[Zel81]{Zeleny81}
M.~Zeleny, editor.
\newblock {\em Autopoiesis: A Theory of Living Organization}.
\newblock North Holland, New York, 1981.

\end{thebibliography}
\bibliographystyle{alpha}

\end{document}